\documentclass[11pt,twoside]{article}
\usepackage{hyperref,xspace,fullpage}
\usepackage{amsfonts,amssymb,amsopn}
\usepackage{vfmacros}
\usepackage{amsmath}
\usepackage{amsthm}

\usepackage{color}
\usepackage[style=alphabetic,backend=bibtex,maxbibnames=20,maxcitenames=6,firstinits=true,doi=false,url=false]{biblatex}
\newcommand*{\citet}[1]{\AtNextCite{\AtEachCitekey{\defcounter{maxnames}{2}}} \textcite{#1}}

\newcommand*{\citep}[1]{\cite{#1}}

\bibliography{vf-allrefs-local,selfbound2refs}

\usepackage{bm}
\usepackage{paralist}

\newtheorem{theorem}{Theorem}[section]
\newtheorem{definition}[theorem]{Definition}
\newtheorem{remark}[theorem]{Remark}
\newtheorem{lemma}[theorem]{Lemma}
\newtheorem{corollary}[theorem]{Corollary}

\newcommand{\infl}{\mathsf{Inf}}

\newcommand{\RR}{\mathbb R}

\newcommand{\sS}{\mathbb{S}}
\newcommand{\NS}{\mathbb{NS}}
\newcommand{\ZZ}{\mathbb Z}

\title{Tight Bounds on $\ell_1$ Approximation and Learning of Self-Bounding Functions}

\author{Vitaly Feldman\thanks{Work done while the author was at IBM Research - Almaden.} \\
Google Brain \and Pravesh Kothari\footnotemark[1] \\
 Carnegie Mellon University
 \and Jan Vondr\'{a}k \\
Stanford University}
\date{}

\begin{document}
\maketitle

\begin{abstract}
We study the complexity of learning and approximation of self-bounding functions over the uniform distribution on the Boolean hypercube $\zo^n$. Informally, a function $f:\zo^n \rightarrow \R$ is self-bounding if for every $x \in \zo^n$, $f(x)$ upper bounds the sum of all the $n$ marginal decreases in the value of the function at $x$. Self-bounding functions include such well-known classes of functions as submodular and fractionally-subadditive (XOS) functions. They were introduced by Boucheron \etal~(2000) in the context of concentration of measure inequalities. % \citep{BoucheronLM:00}.
Our main result is a  nearly tight $\ell_1$-approximation of self-bounding functions by low-degree juntas. Specifically, all self-bounding functions can be $\eps$-approximated in $\ell_1$ by a polynomial of degree $\tilde{O}(1/\eps)$ over $2^{\tilde{O}(1/\eps)}$ variables. We show that both the degree and junta-size are optimal up to logarithmic terms. Previous techniques considered stronger $\ell_2$ approximation and proved nearly tight bounds of  $\Theta(1/\epsilon^{2})$ on the degree and $2^{\Theta(1/\epsilon^2)}$ on the number of variables. Our bounds rely on the analysis of noise stability of self-bounding functions together with a stronger connection between noise stability and $\ell_1$ approximation by low-degree polynomials. This technique can also be used to get tighter bounds on $\ell_1$ approximation by low-degree polynomials and a faster learning algorithm for halfspaces.

These results lead to improved and in several cases almost tight bounds for PAC and agnostic learning of self-bounding functions relative to the uniform distribution. In particular, assuming hardness of learning juntas, we show that PAC and agnostic learning of self-bounding functions have complexity of $n^{\tilde{\Theta}(1/\epsilon)}$.
\end{abstract}

\section{Introduction} \label{sec:intro}
We consider learning and approximation of several classes of real-valued functions over the uniform distribution on the Boolean hypercube $\zo^n$. The most well-studied class of functions that we consider is the class of submodular functions.
A related class of functions is that of {\em fractional subadditive functions}, equivalently known as XOS functions, which generalize monotone submodular functions and have been introduced in the context of combinatorial auctions \citep{LLN06}. XOS functions are also known to have an equivalent definition as Rademacher complexity of a subset of data points for some class of functions \citep{FeldmanVondrak:15}.
It turns out that these classes are all contained in a broader class, that of {\em self-bounding functions}, introduced in the context of concentration of measure inequalities \citep{BoucheronLM:00}. Informally, a function $f$ over $\zo^n$ is $a$-self-bounding if for every $x \in \zo^n$, $a \cdot f(x)$ upper bounds the sum of all the $n$ marginal decreases in the value of the function at $x$. For XOS functions $a=1$ and for submodular\footnote{Technically, self-bounding functions are always non-negative and hence capture only non-negative submodular functions. Submodularity is preserved under shifting of the function and therefore it is sufficient to consider non-negative submodular functions.} $a=2$ ($a$ is omitted when it equals 1). See Sec.~\ref{sec:prelims} for formal definitions and examples of self-bounding functions.

Wide-spread applications of submodular functions have recently inspired the question of whether and how such functions can be learned from random examples (of an unknown submodular function). The question was first formally considered by \citet{BalcanHarvey:12full} who motivate it by learning of valuation functions.
Reconstruction of such functions up to some multiplicative factor from value queries (which allow the learner to ask for the value of the function at any point) was also considered by \citet{GoemansHIM09}. In this work we consider the setting in which the learner gets random and uniform examples of an unknown function $f$ and its goal is to find a hypothesis function $h$ that $\eps$-approximates the unknown function for a given $\eps > 0$. The measure of the approximation error we use is the standard absolute error or $\ell_1$-distance, which equals $\E_{x\sim D}[|f(x)-h(x)|]$.
%Learning of real-valued functions is a classical setting studied in Statistical Learning Theory \citep{Vapnik:98} and can also be seen as a generalization of Valiant's PAC model \citep{Valiant:84} of learning to real-valued functions \citep{Haussler:92}.
While other measures of error, such as $\ell_2$, are often studied in machine learning, there is a large number of scenarios where the expected absolute error is used. For example, if the unknown function is Boolean then learning with $\ell_1$ error is equivalent to learning with Boolean disagreement error \citep{KalaiKMS:08}. In fact, it is known that the complexity of agnostic learning over product distributions in the statistical query model is characterized by how well the Boolean functions can be approximated in $\ell_1$ by low-degree polynomials \citep{DachmanFTWW:15}. Applications of learning algorithms for submodular functions to differentially-private data release require $\ell_1$ error \citep{GuptaHRU:11,CheraghchiKKL:12,FeldmanKothari:14} as does learning of probabilistic concepts (which are concepts expressing the probability of an event) \citep{KearnsSchapire:94}.

Motivated by applications to learning, prior works have also studied a number of natural questions on approximation of submodular and related classes of functions by concisely represented functions. For example, linear functions \citep{BalcanHarvey:12full}, low-degree polynomials \citep{CheraghchiKKL:12,FeldmanVondrak:15}, DNF formulas \citep{RaskhodnikovaYaroslavtsev:13}, decision trees \citep{FeldmanKV:13} and functions of few variables (referred to as {\em juntas}) \citep{FeldmanKV:13,BlaisOSY:13manu,FeldmanVondrak:15,FeldmanVondrak:16}. We survey the prior work in more detail in Section~\ref{sec:prior-work}.
%The uniform (or more generally, product) distribution case is a natural restriction that is widely-studied in the context of learning Boolean functions in the PAC model (\eg \citep{LinialMN:93,ODonnellServedio:07}). It is also the focus of several recent works on learning submodular functions \citep{GuptaHRU:11,CheraghchiKKL:12,RaskhodnikovaYaroslavtsev:13,FeldmanKV:13,FeldmanKothari:13,FeldmanVondrak:16}.

\subsection{Our results}
\label{sec:our-results}
In this work, we provide nearly tight bounds on approximation of self-bounding functions by low-degree polynomials and juntas in the $\ell_1$-norm. The results are obtained via the noise-stability analysis of self-bounding functions. Previous approximation bounds for the uniform distribution relied on bounding $\ell_2$ error that is more convenient to analyze using Fourier techniques. However this approach has so far led to weaker bounds on $\ell_1$ approximation error. Further the known bounds on $\ell_2$ approximation are known to be optimal \citep{FeldmanVondrak:15}. The dependence of the degree and junta size on the error parameter $\eps$ in our bounds is quadratically better (up to a logarithmic term) than bounds which are known for $\ell_2$ error.
%Our results are based on direct analysis of $\ell_1$ error and

%structural and learning results for submodular, XOS and self-bounding functions in the $\ell_1$-norm. The main improvements compared to previous work arise from the fact that we analyze more carefully the Fourier-theoretic properties of these functions using the $\ell_1$-norm, as opposed to the $\ell_2$-norm as has been the case in recent works in this area \citep{CheraghchiKKL:12,FeldmanVondrak:16}. The $\ell_2$-norm is more convenient to use but leads to suboptimal bounds in $\ell_1$, the usual norm of choice for error analysis in machine learning.

\paragraph{Structural results:} Our two key structural results can be summarized as follows.
\begin{theorem}
\label{th:selfbound-junta-l1-bound-simple}
Let $f:\zo^n \rightarrow [0,1]$ be an $a$-self-bounding function and $\eps>0$. For $d = O(a/\epsilon \cdot \log (1/\epsilon))$ there exists a set of indices $I$ of size $2^{O(d)}$ and a polynomial $p$ of degree $d$ over variables in $I$ such that $\|f - p\|_1 \leq \epsilon$.
\end{theorem}
%Moreover the set $I$ is defined as the set of influential variables of $f$.
This result itself is based on a combination of two structural results. The first one gives a degree bound of $O(\frac{a}{\eps} \log \frac{1}{\epsilon})$. Previously, it was known that submodular functions with range $[0,1]$ can be $\epsilon$-approximated by polynomials of degree $O(1/\eps^{2})$ \citep{CheraghchiKKL:12,FeldmanKV:13}. \citet{FeldmanVondrak:16} showed that the same upper bound applies all self-bounding functions and, more generally, all functions of low total influence. More recently, it was shown that this upper bound is tight \citep{FeldmanVondrak:15}. For comparison, as follows from the results in \citep{FeldmanVondrak:15}, for XOS functions there is no significant difference in between $\ell_1$ and $\ell_2$ approximation. In both cases degree $\Theta(1/\eps)$ and junta of size $2^{\Theta(1/\eps)}$ are needed. One natural open problem that is left open is the degree of polynomial necessary to approximate a submodular function in $\ell_1$ norm.

Our proof is based on a new and simple connection between (the appropriately generalized notion of) noise sensitivity of a real-valued function and its approximability by a low-degree polynomial. The key observation here is that the application of the noise operator to a function $f$ that has low noise sensitivity gives a function that is close to $f$ in $\ell_1$ norm. The obtained {\em smoothed} function is much easier to approximate by a low-degree polynomial since its Fourier spectrum decays rapidly with the growth of the degree. This technique is general and also gives a  sharper bound for $\ell_1$ approximation of halfspaces by low-degree polynomials (see Cor.~\ref{cor:halfspaces}). To apply this technique to self-bounding functions we show that noise-sensitivity can be upper bounded using a bound on the total $\ell_1$ influence of all the coordinates on the function. It is known that $a$-self-bounding function have total influence of at most $a$ \citep{FeldmanVondrak:16} and thus we obtain that any $a$-self-bounding functions has bounded noise sensitivity and can be approximated by a degree $O(\frac{a}{\eps} \log \frac{1}{\epsilon})$ polynomial.

The second component of this result builds on the work of \citep{FeldmanVondrak:16}, where it was shown that a classic theorem of \citet{Friedgut:98}, on approximation of Boolean functions by juntas, generalizes to the setting of real-valued functions by including a dependence on $\ell_1$ as well as $\ell_2$-influences of the function. We show that by applying the analysis from \citep{FeldmanVondrak:16} to the smoothed version of $f$ (for which we have better degree bounds) we can obtain approximation by a junta of size $2^{O(a/\epsilon \cdot \log (1/\epsilon))}$. This improves on $2^{O(a/\epsilon^2)}$ bound in \citep{FeldmanVondrak:16} (that holds also for $\ell_2$ error). We note that both of the components also apply to the more general class of functions with low total $\ell_1$ influence.

We then study the effect of the noise operator on self-bounding functions in more detail. We demonstrate that the smoothed version is noise stable even in the stronger point-wise sense: for every $x$, the smoothed function at $x$ cannot be much smaller than $f(x)$. This result generalizes a similar result from \citep{CheraghchiKKL:12} for submodular functions.  Such stability implies that for every non-negative $a$-self-bounding function $f$, $\|f\|_1 \geq \frac{1}{3^a} \|f\|_\infty$ (see Lemma \ref{lem:hypercontractive}). This has been known for submodular \citep{FeigeMV07} and XOS \citep{Feige:06} functions (with a constant $a$) and, together with approximation by a junta, can be used to obtain a learning algorithm with multiplicative approximation guarantees for all $a$-self-bounding functions \citep{FeldmanVondrak:16}.

% and was useful in result generalizes analogous results for

\paragraph{Algorithmic applications:}
It is easy to exploit our structural results in existing learning algorithms to obtain better running time and sample complexity bounds. We describe one of these results here and some additional ones in Section \ref{sec:applications}. Specifically, we give an algorithm for learning all $a$-self-bounding functions relative to the uniform distribution in the challenging agnostic framework. An agnostic learning algorithm for a class of functions $\C$ is an algorithm that given random examples of {\em any} function $f$ finds a hypothesis $h$ whose error is at most $\eps$-greater than the error of the best hypothesis in $\C$ (see \citep{KearnsSS:94} for the Boolean case).
\begin{theorem}
\label{th:agn-learn-self-bound-l1-simple}
Let $\C_a$ be the class of all $a$-self-bounding functions from $\zo^n$ to $[0,1]$.  There exists an algorithm $\A$ that given $\eps > 0$ and access to random uniform examples of any real-valued $f$, with probability at least $2/3$, outputs a function $h$, such that $\|f-h\|_1 \leq \Delta + \epsilon$, where $\Delta = \min_{g\in \C_a}\{\|f-g\|_1\}$. Further, $\A$ runs in time $ n^{\tilde{O}(a/\eps)}$ and uses $2^{\tilde{O}(a^2/\eps^{2})} \log n$ examples.
\end{theorem}
This algorithm is based on polynomial $\ell_1$ regression with an additional constraint on the spectral norm of the solution to obtain a stronger sample complexity bound \citep{FeldmanVondrak:16}. The best previous bound of $n^{O(a/\eps^2)}$ time and $2^{O(a^2/\eps^{4})} \log n$ examples follows from the results in \citep{FeldmanVondrak:16} for function of low total influence. %For submodular functions the sample complexity can be further strengthened to  $2^{\tilde{O}(1/\eps)} \log n$ examples (see Cor.~\ref{cor:selfbound-junta-l1-spectral-bound}).
%Further details of algorithmic applications are given in Section~\ref{sec:applications}.
%We discuss the definitions and an application relevant to the problem of testing submodular functions in Section \ref{}.

\paragraph{Lower bounds:}
We prove that $a$-self-bounding functions require degree $\Omega(a/\eps)$ to $\eps$-approximate in $\ell_1$ distance (see Cor.~\ref{cor:degree-lower-bound-sb}). A construction of a parity function correlated with a submodular function in \citep{FeldmanKV:13} also implies that even submodular functions require polynomials of degree $\Omega(\eps^{-2/3})$ to $\eps$-approximate in $\ell_1$.

In \citep{FeldmanVondrak:16} it is shown that XOS functions require a junta of size $2^{\Omega(1/\eps)}$ to $\eps$-approximate (however submodular functions admit approximation by exponentially smaller juntas \citep{FeldmanVondrak:16}). This also implies $2^{\Omega(a/\eps)}$ lower bound on junta size for $a$-self-bounding functions (see Lem.~\ref{lem:lifting-trick}). Therefore our structural results are essentially tight for self-bounding functions.

We then show that our agnostic learning algorithm for $a$-self-bounding function is nearly optimal. In fact, even PAC learning of non-monotone $a$-self-bounding functions requires time $n^{\Omega(a/\eps)}$ assuming hardness of learning $k$-term DNF to accuracy $1/4$ in time $n^{\Omega(k)}$. This is in contrast to the submodular \citep{FeldmanKV:13,FeldmanVondrak:16} and monotone self-bounding cases (Thm.~\ref{thm:pac-learn-monotone}).
\begin{theorem}
\label{th:self-bound-pac-is-hard}
For every $a \geq 1$, if there exists an algorithm that PAC learns $a$-self-bounding functions with range $[0,1]$ to $\ell_1$ error of $\eps >0$ in time $T(n,1/\eps)$ then there exists an algorithm that PAC learns $k$-DNF formulas to accuracy $\eps'$ in time $T(n, k/(a \cdot \eps'))$ for some fixed constant $c$.
\end{theorem}
To prove this hardness results we show that a $k$-DNF formula (of any size) is a $k$-self-bounding function. Using an additional ``lifting" trick we can also embed $k$-DNF formulas into $a$-self-bounding functions for any $a \geq 1$. Note that any $k$-junta can be computed by a $k$-DNF formula. Learning of DNF expressions is a well-studied problem in learning theory but there are no algorithms for this problem better than the trivial $O(n^k)$ algorithm, even for a constant $\eps'=1/4$. The (potentially simpler) problem of learning $k$-juntas is also considered very hard \citep{BlumLangley:97,Blum:03op}. Until recently, the only non-trivial algorithm for the problem was the $O(n^{0.7k})$-time algorithm by \citet{MosselOS:04}. The best known upper bound is $O(n^{0.6k})$ and was given in the recent breakthrough result of \citet{ValiantG:12}. Learning of $k$-juntas is also known to have complexity of $n^{\Omega(k)}$ for all statistical query algorithms \citep{BlumFJ+:94}. Theorem \ref{th:self-bound-pac-is-hard} implies that PAC learning of $a$-self-bounding functions in time $n^{o(a/\eps)}$ would lead to a $n^{o(k)}$ algorithm for learning $k$-DNF to any constant accuracy and, in particular, an algorithm for PAC learning $k$-juntas in time $n^{o(k)}$. We note that the dependence on $a/\eps$ in our lower bound matches our upper bound up to a logarithmic factor.

Finally, we remark that our reduction to learning of $k$-DNF also implies that PAC learning of $a$-self-bounding functions requires at least $2^{\Omega(a/\eps)}$ random examples or even stronger value queries (see Cor.~\ref{cor:inf-lower-bound-sb}). Therefore sample complexity bounds we give are also close to optimal. Further details of lower bounds are given in Section~\ref{sec:lower-bounds}.

\subsection{Related work}
\label{sec:prior-work}
Below we briefly mention some of the other related work. We direct the reader to \citep{BalcanHarvey:12full} and \citep{FeldmanVondrak:15} for more detailed surveys. Balcan and Harvey study learning of submodular functions without assumptions on the distribution and also require that the algorithm output a value which is within a multiplicative approximation factor of the true value with probability $\geq 1 - \eps$ (the model is referred to as {\em PMAC learning}). This is a very demanding setting and indeed one of the main results in \citep{BalcanHarvey:12full} is a factor-$\sqrt[3]{n}$ inapproximability bound for submodular functions. This notion of approximation is also considered in subsequent works of  \citet{BadanidiyuruDFKNR:12} and \citet{BalcanCIW:12} where upper and lower approximation bounds are given for other related classes of functions such as XOS and subadditive.
We emphasize that these strong lower bounds rely on a very specific distribution concentrated on a sparse set of points, and show that this setting is very different from uniform/product distributions which are the focus of this paper.

%For product distributions, Balcan and Harvey show that 1-Lipschitz submodular functions of minimum nonzero value at least $1$ have concentration properties implying a PMAC algorithm providing an $O(\log \frac{1}{\eps})$-factor approximation except for an $\epsilon$-fraction of points, using $O(\frac{1}{\epsilon} n \log n)$ samples \citep{BalcanHarvey:12full}. In our setting,
%we have no assumption on the minimum nonzero value, and we are interested in the additive $\ell_1$-error rather than multiplicative approximation. One way to view the approach of Balcan and Harvey which relies on concentration of $f$ around the expectation, is as approximation by a degree-$0$ polynomial.

\citet{GuptaHRU:11} motivate learning of submodular functions over the uniform distribution by problems in differentially-private data release. They show that submodular functions with range $[0,1]$
are $\epsilon$-approximated by a collection of $n^{O(1/\epsilon^2)}$ $\epsilon^2$-Lipschitz submodular functions. Each $\epsilon^2$-Lipschitz submodular function
can be $\epsilon$-approximated by a constant. This leads to a learning algorithm running in time $n^{O(1/\epsilon^2)}$, which however requires value oracle access to the target function, in order to build the collection.

The work of \citet{CheraghchiKKL:12} studies approximations of submodular functions by low-degree polynomials. They prove that any submodular function (of unit norm) can be $\epsilon$-approximated in $\ell_1$ by a polynomial of degree $O(1/\epsilon^2)$. This leads again to an $n^{O(1/\epsilon^2)}$-time algorithm, but one which requires only random examples and works even in the agnostic setting. The main tool used in this work is the notion of noise stability. \citet{FeldmanVondrak:16} studied approximation of submodular, XOS and self-bounding functions by juntas. Their main result shows that submodular functions can be approximated in $\ell_2$ by a junta of size $\tilde O(1/\eps^2)$ and further that all self-bounding functions can be approximated by a junta of size $2^{O(1/\epsilon^2)}$.

Subsequently, \citet{FeldmanVondrak:15} have obtained tight bounds on the degree of a polynomial that is sufficient to approximate any function in each of these function classes in $\ell_2$ norm. Specifically, they showed $\tilde\Theta(\eps^{-4/5})$ bound for submodular functions, $\Theta(1/\eps)$ bound for XOS functions and a matching lower bound of $\Omega(1/\eps^2)$ for self-bounding functions. The degree bound for XOS functions also implies an upper bound of $2^{O(1/\eps)}$ on the size of the junta sufficient to approximate (in $\ell_2$) any XOS function.

\citet{RaskhodnikovaYaroslavtsev:13} consider learning and testing of submodular functions taking values in the range $\{0,1,\ldots,k\}$ (referred to as {\em pseudo-Boolean}). The error of a hypothesis in their framework is the probability that the hypothesis disagrees with the unknown function. They build on the approach from \citep{GuptaHRU:11} to show that pseudo-Boolean submodular functions can be expressed as $2k$-DNF and then give a $\poly(n) \cdot k^{O(k \log{k/\eps})}$-time PAC learning algorithm using value queries. \citet{BlaisOSY:13manu} proved existence of a junta of size $(k \log(1/\eps))^{O(k)}$ and used it to give an algorithm for testing submodularity using $(k \log(1/\eps))^{\tilde{O}(k)}$ value queries.
\citet{FeldmanVondrak:16} and, more recently, \citet{BlaisB17} have studied testing of various type of valuation functions showing that approximation by a junta can be exploited to get efficient testing algorithms.

\section{Preliminaries} \label{sec:prelims}

\subsection{Submodular, subadditive and self-bounding functions}
\label{sec:classes}

In this section, we define the relevant classes of functions. We refer the reader to \citep{Vondrak10,FeldmanVondrak:16} for more details.

\begin{definition}
A set function $f:2^N \rightarrow \RR$ is
\begin{compactitem}
 \item monotone, if $f(A) \leq f(B)$ for all $A \subseteq B \subseteq N$.
 \item submodular, if $f(A \cup B) + f(A \cap B) \leq f(A) + f(B)$ for all $A,B \subseteq N$.
 \item fractionally subadditive, if $f(A) \leq \sum \beta_i f(B_i)$ whenever $\beta_i \geq 0$
 and $\sum_{i:a \in B_i} \beta_i \geq 1 \ \forall a \in A$.
\end{compactitem}
\end{definition}

Submodular functions are not necessarily nonnegative, but in many applications (especially when considering multiplicative approximations), this is a natural assumption. Fractionally subadditive functions are nonnegative by definition (by considering $A=B_1, \beta_1 > 1$). In this paper we work exclusively with functions $f:2^N \rightarrow \RR_+$.

Next, we introduce {\em $a$-self-bounding functions}. Self-bounding functions were defined by
\citet{BoucheronLM:00} as a unifying class of functions that enjoy strong ``dimension-free" concentration properties. Currently this is the most general class of functions known to satisfy such concentration bounds.
Self-bounding functions are defined generally on product spaces $X^n$; here we restrict our attention to the hypercube,
so the reader can assume that $X = \zo$.
We identify functions on $\zo^n$ with set functions on $N = [n]$ in a natural way.
Here we define a somewhat more general class of $a$-self-bounding functions, following \citep{McDiarmidR06}.

\begin{definition}
\label{def:self-bound}
A function $f:\zo^n \rightarrow \RR$ is $a$-self-bounding, if
%there are functions $f_i:X^{n-1} \rightarrow \RR$ such that if we denote
%$x^{(i)} = (x_1,\ldots,x_{i-1},x_{i+1},\ldots,x_n)$, then
for all $x \in \zo^n$ and $i \in [n]$,
$$ f(x) - \min_{x_i} f(x) \leq 1 $$
and
$$ \sum_{i=1}^{n} (f(x) - \min_{x_i} f(x)) \leq a f(x).$$
\end{definition}

Useful properties of $a$-self-bounding functions that are easy to verify is that they are closed under taking $\max$ operation and closed under taking convex combinations. A particular example of a self-bounding function (related to applications of Talagrand's inequality) is a function with the property of {\em small certificates}:
$f:X^n \rightarrow \ZZ_+$ has small certificates, if it is 1-Lipschitz and whenever $f(x) \geq k$, there is a set of coordinates $S \subseteq [n]$, $|S|=k$, such that if $y|_S = x|_S$, then $f(y) \geq k$. Such functions often arise in combinatorics, by defining $f(x)$ to equal the maximum size of a certain structure appearing in $x$. In Section \ref{sec:lower-bounds} we also show that $k$-DNF formulas are $k$-self-bounding.

\subsection{Fourier analysis on the Boolean cube}
The $\ell_1$ and $\ell_2$-norms of a $f:\zon\rightarrow \RR$ are defined by $\|f\|_1 =  \E_{x \sim \U} [|f(x)|]$ and $\|f\|_2 =  (\E_{x \sim \U} [f(x)^2])^{1/2}$, respectively, where $\U$ is the uniform distribution over $\zon$. In what follows all probabilities and expectations are relative to $\U$ unless explicitly specified otherwise.

We rely on the standard Fourier transform representation of real-valued functions over $\zon$ as linear combinations of parity functions.
For $S \subseteq [n]$, the parity function $\chi_S:\zon \rightarrow \on$ is defined by
$ \chi_S(x) = (-1)^{\sum_{i \in S} x_i}$. The Fourier expansion of $f$ is given by $ f(x) = \sum_{S \subseteq [n]} \hat{f}(S) \chi_S(x).$ The degree of the highest degree non-zero Fourier coefficient of $f$ is referred to as the {\em Fourier degree} of $f$.
Note that Fourier degree of $f$ is exactly the polynomial degree of $f$ when viewed over $\on^n$ instead of $\zon$ and therefore it is also equal to the polynomial degree of $f$ over $\zon$. Let $f: \zon \rightarrow \RR$ and $\hat{f}: 2^{[n]} \rightarrow \RR$ be its Fourier transform.

\begin{definition}[The noise operator]
For $\rho \in [-1,+1], x \in \zo^n$, we define a distribution $N_\rho(x)$ over $y \in \zo^n$ by letting $y_i = x_i$ with probability $\frac{1+\rho}{2}$ and $y_i = 1-x_i$ with probability $\frac{1-\rho}{2}$, independently for each $i$.
The noise operator $T_\rho$ acts on functions $f:\zo^n \rightarrow \RR$, and is defined by
$$ (T_\rho f)(x) = \E_{y \sim N_\rho(x)}[f(y)].$$ The noise stability of $f$ at noise rate $\rho$ is
$$ \sS_\rho(f) = \langle f,T_\rho f \rangle = \E[f(x) T_\rho f(x)].$$
\end{definition}

%\paragraph{Note.} In the case of $\Pi$ being the uniform distribution,
%$y \sim N_\rho(x)$ has $y_i = x_i$ with probability $\frac{1+\rho}{2}$ and $y_i = 1-x_i$ with probability $\frac{1-\rho}{2}$, independently for each $i$.

%\begin{lemma}
In terms of Fourier coefficients, the noise operator acts as $\widehat{T_\rho f}(S) = \rho^{|S|} \hat{f}(S)$. Therefore, noise stability can be written as
$ \sS_\rho(f) = \sum_{S \subseteq [n]} \rho^{|S|} \hat{f}^2(S).$
Finally, we define noise sensitivity that generalizes the notion of noise sensitivity for Boolean functions.
\begin{definition}[Noise sensitivity]
\label{def:noise-sensitivity}
For $\delta \in [0,1]$ and a function $f:\zo^n \rightarrow \RR$, the noise sensitivity of $f$ at $\delta$ is
$\NS_\delta(f) = \fr{2}\|f - T_{1-2\delta} f\|_1 =  \fr{2}\E[|f(x) - T_{1-2\delta} f(x)|]$.
\end{definition}
We keep the factor $1/2$ in the definition for consistency with the Boolean case. In the Boolean case noise sensitivity has the following relationship to noise stability (\eg \citep{ODonnell13}):
\equn{
\NS_\delta(f) = \fr{2} \left(1-\sS_{1-2\delta}(f)\right).
}

%In particular (for $\rho=1$), $\|f\|_2^2 = \sum_{S \subseteq [n]} \hat{f}^2(S)$.
%\end{lemma}

\begin{definition}[Discrete derivatives]
For $x \in \zo^n$, $b \in \zo$ and $i \in n$ let $x_{i\leftarrow b}$ denote the vector in $\zo^n$ that equals to $x$  with $i$-th coordinate set to $b$. For a real-valued $f:\zo^n \rightarrow \R$ and indices $i,j\in [n]$ we define, $\partial_i f(x) = \frac12(f(x_{i\leftarrow 1}) - f(x_{i\leftarrow -1}))$.
We also define $\partial_{i,j} f(x) = \partial_i \partial_j f(x)$.
\end{definition}

Observe that
$\partial_i f(x) = \sum_{S \ni i}\hat{f}(S)\chi_{S\setminus\{i\}}(x)$, and
$\partial_{i,j} f(x) = \sum_{S \ni i,j}\hat{f}(S)\chi_{S\setminus\{i,j\}}(x)$.
%For a monotone (non-decreasing) function, $\partial_i f(x) \geq 0$. For a submodular function,
%$\partial_{i,j} f(x) \leq 0$, by considering the submodularity condition for $x_{i \leftarrow 0, j \leftarrow 0}$, $x_{i \leftarrow 0, j \leftarrow 1}$, $x_{i \leftarrow 1, j \leftarrow 0}$, and  $x_{i \leftarrow 1, j \leftarrow 1}$.

We use several notions of {\em influence} of a variable on a real-valued function which are based on the standard notion of influence for Boolean functions (\eg \citep{Ben-OrLinial:85,KahnKL:88}).
\begin{definition}[Influences]
For a real-valued $f:\zo^n \rightarrow \RR$, $i \in [n]$, and $\kappa \geq 0$ we define the {\em $\ell_\kappa^\kappa$-influence} of variable $i$ as $\infl^\kappa_i(f) = \|\partial_i f\|_\kappa^\kappa = \E[|\partial_i f|^\kappa]$. We define $\infl^\kappa(f) = \sum_{i\in[n]} \infl^\kappa_i(f)$ and refer to it as the {\em total $\ell_\kappa^\kappa$-influence} of $f$. %For a boolean function $f:\zon\rightarrow \zo$, $\infl(f)$ is defined as $2\infl^1(f)$ and is also referred to as {\em average sensitivity}.
\end{definition}

%%%%%%%%%%%%%%%%%%%%%%%%%

\section{Structural results}
\label{sec:structural}
\subsection{Approximation of low-sensitivity functions by low-degree polynomials}
In this section we demonstrate a simple approach that allows to approximate low noise-sensitive functions in $\ell_1$ norm and also show that noise sensitivity of a function can be upper-bounded by its $\ell_1$ influence.

Our approach is based on an observation that if a function is close to its noisy version in $\ell_1$ norm then it is well-approximated by a low-degree polynomial.
\begin{lemma}
\label{lem:noise2approx}
For every function $f:\zo^n \rightarrow \RR$, every $\epsilon>0$ and $\delta \in (0,1]$ there exists a multilinear polynomial $p$ of degree
$d = \lceil \frac{1}{2\delta} \log \frac{1}{\epsilon} \rceil$ such that
$$ \|f - p\|_1 \leq \epsilon \|f\|_2 + 2 \cdot \NS_\delta(f).$$
In particular, the polynomial can be chosen as $p(x) = \sum_{|S|<d} (1-2\delta)^{|S|} \hat{f}(S) \chi_S(x)$.
\end{lemma}
\begin{proof}
Let $\rho = 1-2\delta$. We can estimate the tail of the Fourier expansion as follows: For any $d$, define $f_{<d}(x)
 = \sum_{S: |S|<d} \hat{f}(S) \chi_S(x)$, a polynomial of degree at most $d$.
Then,  since $T_\rho f(x) = \sum_{S \subseteq [n]} \rho^{|S|} \hat{f}(S) \chi_S(x)$,
we get
\begin{equation}
\label{eq:noise-2}
\|T_\rho f_{<d} - T_\rho f\|_1 = \left\| \sum_{S:|S|\geq d} \rho^{|S|} \hat{f}(S) \chi_S \right\|_1
 \leq \left\| \sum_{S:|S|\geq d} \rho^{|S|} \hat{f}(S) \chi_S \right\|_2
 \leq \rho^d \|f\|_2.
 \end{equation}
Taking $d = \lceil \frac{1}{2\delta} \log \frac{1}{\epsilon} \rceil$ we get that such that $\|T_\rho f_{<d} - T_\rho f\|_1 \leq (1-2\delta)^d \cdot \|f\|_2  \leq \epsilon \|f\|_2$.
Now, by Definition \ref{def:noise-sensitivity}, we have that $\|f - T_{1-2\delta} f\|_1 = 2 \cdot \NS_{\delta}(f)$. The lemma now follows by the triangle inequality.
\end{proof}

Next we observe that the total $\ell_1$ influence of a function can be used to derive an upper-bound on its noise sensitivity.
\begin{lemma}
\label{lem:infl2noise}
For every function $f:\zo^n \rightarrow \RR$ and $\delta \in [0,1]$,
$$\NS_\delta(f) \leq \delta \cdot \infl^1(f).$$
\end{lemma}
\begin{proof}
For every $t= 0,1,\ldots,n$ we define a distribution $N^{1:t}_{1-2\delta}(x)$ over $y \in \zo^n$ by letting $y_i = x_i$ with probability $1-\delta$ and $y_i = 1-x_i$ with probability $\delta$, independently for each $i \leq t$, while  for all $i > t$, $y_i=x_i$. Note that $N^{1:0}_{1-2\delta}(x)$ is always equal to $x$ and $N^{1:n}_{1-2\delta}(x)$ is exactly $N_{1-2\delta}(x)$.
We also define a distribution $N^{t}_{1-2\delta}(x)$ over $y \in \zo^n$ by letting $y_y = x_t$ with probability $1-\delta$ and $y_t = 1-x_t$ with probability $\delta$, while  for all $i \neq t$, $y_i=x_i$.

Now,
\alequn{
\NS_\delta(f) &= \fr{2} \cdot \E[|f(x) - T_{1-2\delta} f(x)|] = \fr{2} \cdot \E\left[\left|f(x) - \E_{y\sim N_{1-2\delta}(x)}[f(y)]\right|\right] \\
&\leq \fr{2} \sum_{t=1}^{n} \E\left[\left|\E_{y\sim N^{1:t-1}_{1-2\delta}(x)}[f(y)] - \E_{y\sim N^{1:t}_{1-2\delta}(x)}[f(y)]\right|\right] \\ &\leq \fr{2} \sum_{t=1}^{n} \E\left[\E_{y\sim N^{1:t-1}_{1-2\delta}(x)}\left[\left|f(y)- \E_{z\sim N^t_{1-2\delta}(y)}[f(z)]\right|\right]\right] \\
\\ & =  \fr{2} \sum_{t=1}^{n} \E_{y\sim \U}\left[\left|f(y)- \E_{z\sim N^t_{1-2\delta}(y)}[f(z)]\right|\right]
 \\ & =  \fr{2} \sum_{t=1}^{n} \E_{y\sim \U}\left[2\delta \left|\partial_t f(y)\right|\right] = \delta \sum_{t=1}^{n}  \infl_t^1(f) = \delta \cdot \infl^1(f).
}
\end{proof}

An immediate corollary of Lemmas \ref{lem:noise2approx} and \ref{lem:infl2noise} is that any function of low total $\ell_1$ influence can be well-approximated by a low-degree polynomial:
\begin{corollary}
For every function $f:\zo^n \rightarrow \RR$ such that $\|f\|_2 \leq 1$ and every $\epsilon>0$ there exists a multilinear polynomial $p$ of degree $d = \lceil \frac{2\cdot \infl^1(f)}{\eps} \log \frac{2}{\epsilon} \rceil$ such that
$ \|f - p\|_1 \leq \epsilon$.
\end{corollary}
It follows easily from the definition of self-bounding functions that they have low total $\ell_1$-influence.
\begin{lemma}[\citep{FeldmanVondrak:16}, Lemma 4.2]
\label{lem:submod}
Let $f:\zo^n\rightarrow \R_+$ be an $a$-self-bounding function. Then $\infl^1(f) \leq a\cdot \|f\|_1$.
In particular, for $f:\zo^n\rightarrow [0,1]$, $\infl^1(f) \leq a$.
\end{lemma}

Therefore we obtain that self-bounding functions are well-approximated by low-degree polynomials.
\begin{theorem}
\label{thm:selfbounding-approx}
For every $a$-self-bounding function $f:\zo^n \rightarrow \RR_+$ and every $\epsilon>0$, there exists a multilinear polynomial $p$ of degree
$d = \lceil \frac{2a}{\epsilon} \log \frac{2}{\epsilon} \rceil$ such that
$$ \|f - p\|_1 \leq \epsilon \|f\|_2.$$
In particular, the polynomial can be chosen as $p(x) = \sum_{|S|<d} \rho^{|S|} \hat{f}(S) \chi_S(x)$, for $\rho = 1 - \frac{\epsilon}{2a}$.
\end{theorem}

\paragraph{Application to approximation and learning of halfspaces.}
We now briefly show that our approach can also be used to obtain sharper bounds on $\ell_1$-approximation of halfspaces by low-degree polynomials. Recall that a halfspace is a Boolean function expressible as $\sgn(\sum_{i\in [n]} a_i x_i - a_0)$ for some real values  $a_0,a_1,\ldots,a_n$. Halfspaces are known to be noise-stable. Specifically, \citet{KalaiKMS:08} proved that for every halfspace $f$ and $\delta > 0$, $\NS_\delta(f) \leq 8.8 \cdot \sqrt{\delta}$. Using this fact they showed that any halfspace can be $\eps$-approximated in $\ell_2$ norm by a polynomial of degree $O(1/\eps^4)$ and gave an agnostic learning algorithm for learning halfspaces over the uniform distribution that runs in time $n^{O(1/\eps^4)}$. For $\ell_1$ norm approximation the best previously known bound is $O(\log^2(1/\eps)/\eps^2)$ and was given by \citet{DiakonikolasGJSV10} (note however that their result is substantially more involved and gives a stronger notion of approximation that is necessary for fooling halfspaces). By plugging the upper bound on noise sensitivity into our Lemma \ref{lem:noise2approx} with $\delta = (4\cdot 8.8)^{-2} \cdot \eps^2$ we obtain the following corollary:
\begin{corollary}
\label{cor:halfspaces}
For every halfspace $f$ and every $\epsilon>0$, there exists a multilinear polynomial $p$ of degree
$d = O(\log(1/\eps)/\eps^2))$ such that $\|f - p\|_1 \leq \epsilon$.
\end{corollary}
We note that the agnostic learning algorithm for halfspaces in \citep{KalaiKMS:08} requires only $\ell_1$ approximation. Therefore our result implies that halfspaces are agnostically learnable over the uniform distribution in time $n^{O(\log(1/\eps)/\eps^2))}$.

\subsection{Noise stability of self-bounding functions}
\label{sec:selfbounding}
In this section, we study the action of the noise operator on a self-bounding function in more detail. Specifically, we show that self-bounding functions are noise-stable point-wise. This result strengthens and generalizes a similar one proved for submodular functions in \citep{CheraghchiKKL:12}. It allows us to derive additional properties of self-bounding functions useful for their approximation and learning.

\begin{lemma}
\label{lem:pointwise-bound}
For any $a$-self-bounding function $f:\zo^n \rightarrow \RR$ under the uniform distribution, $a \geq 1$, and any $\rho \in [-1,+1], x \in \zo^n$,
$$ T_\rho f(x) \geq \left(\left(1 - \frac{1-\rho}{2(1-\frac{a-1}{n})}\right)_+\right)^a f(x).$$
\end{lemma}

(We denote by $(z)_+ = \max \{z,0\}$ the positive part of a real number.)

\begin{proof}
First, let us observe that the statement of the lemma is invariant under flipping the hypercube $\zo^n$ along any coordinate: the notion of $a$-self-bounding functions does not change, the action of the noise operator does not change, and the conclusion of the lemma does not change either. So we can assume without loss of generality that $x = (0,0,\ldots,0)$. We also identify points in $\zo^n$ with sets $S \subseteq [n]$ by considering $S = \{i: x_i = 1\}$.

Let us average the values of $f$ over levels of sets of constant $|S|$, and define
$$ \phi(t) = \E_{|S|=t}[f(S)] = \frac{1}{{n \choose t}} \sum_{S: |S|=t} f(S).$$
In particular, $\phi(0) = f(\emptyset) = f(x)$.
We claim the following: for every $t = 0,1,\ldots,n$,
\begin{equation}
\label{eq:level-bound}
\phi(t) \geq \left( \left( 1 - \frac{t}{n-a+1} \right)_+ \right)^a \phi(0).
\end{equation}
Intuitively, if $f(x)$ is a point of high value, the value cannot drop off too quickly as we move away from $x$.
If we prove (\ref{eq:level-bound}), then we are done: for $x = (0,0,\ldots,0)$, $T_\rho f(x)$ is an expectation of $f(S)$ over a distribution where the sets on each level appear with the same probability, namely
$$T_\rho f(x) = \sum_{i=0,1,\ldots,n} \left(\frac{1-\rho}{2}\right)^i \cdot  \left(\frac{1+\rho}{2}\right)^{n-i} \cdot {n \choose i}\cdot  \phi(i) .$$
The expected cardinality of a set sampled from this distribution is $\E[|S|] = \frac{1-\rho}{2} n$. By convexity of the bound in (\ref{eq:level-bound}) and Jensen's inequality, we obtain
$$ T_\rho f(x) \geq \left( \left( 1 - \frac{\frac{1-\rho}{2} n}{n - a + 1} \right)_+ \right)^a f(x)
 = \left( \left(1 - \frac{1-\rho}{2(1 - \frac{a-1}{n})} \right)_+ \right)^a f(x).$$
So it remains to prove (\ref{eq:level-bound}).

We proceed by induction. For $t = 0$, the claim is trivial. Let us assume it holds for $t$, and consider a set $S$, $|S| = t$. We also assume that $t < n-a+1$, because otherwise the claim is trivial (recall that $f$ and hence $\phi$ is nonnegative). By the $a$-self-bounding  property, we have
$$ a f(S) \geq \sum_{i=1}^{n} (f(S) - \min \{ f(S+i), f(S-i) \}) \geq \sum_{i \in [n] \setminus S} (f(S) - f(S+i)) .$$
Note that $|[n] \setminus S| = n-t$. By rearranging this inequality, we get
$$ (n-t-a) f(S) \leq \sum_{i \in [n] \setminus S} f(S+i).$$
Now let us add up this inequality over all $S$ of size $|S|=t$:
$$ (n-t-a) \sum_{|S|=t} f(S) \leq \sum_{|S|=t, i \notin S} f(S+i)
 = (t+1) \sum_{|S'|=t+1} f(S') $$
because every set $S'$ of size $t+1$ appears $t+1$ times in the penultimate summation.
Expressing this inequality in terms of $\phi(t)$,
we get
$$ (n-t-a) {n \choose t} \phi(t) \leq (t+1) {n \choose t+1} \phi(t+1),$$
or equivalently (for $t < n-a$)
$$ \phi(t) \leq \frac{n-t}{n-t-a} \phi(t+1).$$
We replace this by a slightly weaker bound: $\phi(t) \leq (\frac{n-t-a+1}{n-t-a})^a \phi(t+1)$. To see why this holds, consider
$(\frac{n-t-a+1}{n-t-a})^a = (1 + \frac{1}{n-t-a})^a \geq 1 + \frac{a}{n-t-a} = \frac{n-t}{n-t-a}$.

By the inductive hypothesis (\ref{eq:level-bound}), we assume $\phi(t) \geq (\frac{n-a+1-t}{n-a+1})^a \phi(0)$. So we obtain
$$ \left(\frac{n-a+1-t}{n-a+1}\right)^a \phi(0) \leq \left( \frac{n-t-a+1}{n-t-a} \right)^a \phi(t+1).$$
This holds for $t < n-a$, hence proving the bound in (\ref{eq:level-bound}) for $t < n-a+1$:
$$ \phi(t) \geq \left( \frac{n-a-t+1}{n-a+1} \right)^a \phi(0) = \left( 1 - \frac{t}{n-a+1} \right)^a \phi(0).$$
As we mentioned, (\ref{eq:level-bound}) is trivially true for $t \geq n-a+1$.
\end{proof}

\begin{corollary}
For any $a$-self-bounding function $f:\zo^n \rightarrow \RR$ under the uniform distribution,
the noise stability with noise parameter $\rho$ is
$$ \sS_\rho(f) \geq \left(1 - \frac{1-\rho}{2(1-\frac{a-1}{n})}\right)^a \|f\|_2^2.$$
\end{corollary}

In particular, for $a=1$ (self-bounding functions), we obtain $\sS_\rho(f) \geq \frac{1+\rho}{2} \|f\|_2^2$. In \citep{CheraghchiKKL:12}, an analogous bound on noise stability is used to derive an agnostic learning algorithm (over the uniform distribution) with excess $\ell_1$-error $\epsilon$ in time $n^{O(1/\epsilon^2)}$.

\eat{
This result is (implicitly) based on an $\ell_1$-approximation of a submodular function by a low-degree polynomial. Here, we prove by a more direct argument that every $a$-self-bounding function $f$ is well approximated in $\ell_1$ by a low-degree polynomial, with an improved dependence on $\epsilon$ (which implies agnostic learning in time $n^{O(\frac{1}{\epsilon} \log \frac{1}{\epsilon})}$).

Now we are ready to prove Theorem~\ref{thm:selfbounding-approx}.

\medskip

\begin{proof}
Let $\rho \in [-1,+1]$ be a noise parameter to be determined later.
Denote $\tau = (1 - \frac{1-\rho}{2(1 - (a-1)/n)})^a$.
By Lemma~\ref{lem:pointwise-bound}, we have $T_\rho f(x) \geq \tau f(x)$. Therefore, using the triangle inequality,
\begin{equation}
\label{eq:noise-1}
 \left\| T_\rho f - f \right\|_1 \leq \left\| T_\rho f - \tau f \right\|_1 + \left\| \tau f - f \right\|_1 \\
 = \E[T_\rho f(x) - \tau f(x)] + \E[f(x) - \tau f(x)] = 2(1-\tau) \left\| f \right\|_1
\end{equation}
using the facts that $T_\rho f - \tau f$ and $f - \tau f$, as well as $f$ itself, are nonnegative functions and that $\E[T_\rho f(x)] =  \E[f(x)]$.
On the other hand, since $T_\rho f(x) = \sum_{S \subseteq [n]} \rho^{|S|} \hat{f}(S) \chi_S(x)$,
we can estimate the tail of the Fourier expansion as follows: For any $d$, define $f_{<d}(x)
 = \sum_{S: |S|<d} \hat{f}(S) \chi_S(x)$, a polynomial of degree at most $d$.
Then
\begin{equation}
\label{eq:noise-2}
\|T_\rho f_{<d} - T_\rho f\|_1 = \left\| \sum_{S:|S|\geq d} \rho^{|S|} \hat{f}(S) \chi_S \right\|_1
 \leq \left\| \sum_{S:|S|\geq d} \rho^{|S|} \hat{f}(S) \chi_S \right\|_2
% = \left( \sum_{S: |S|\geq d} \rho^{2|S|} \hat{f}^2(S) \right)^{1/2}
 \leq \rho^d \|f\|_2.
% \leq \rho^d 3^a \|f\|_1.
 \end{equation}
% \|T_\rho f_{<d} - T_\rho f\|_1 \leq \rho^d \|f\|_2 \leq \rho^d 3^a \|f\|_1.
Combining (\ref{eq:noise-1}) and (\ref{eq:noise-2}) by the triangle inequality, we obtain
\begin{equation}
\label{eq:noise-3}
\left\| T_\rho f_{<d} - f \right\|_1 \leq
\left\|T_\rho f_{<d} - T_\rho f \right\|_1 + \left\|T_\rho f - f \right\|_1
 \leq \rho^d \|f\|_2 + 2(1 - \tau) \|f\|_2.
\end{equation}
We also used the fact that $\|f\|_1 \leq \|f\|_2$.
The function $T_\rho f_{<d}$ is our promised polynomial $p$.
Now, let us set $d = \lceil \frac{2a}{\epsilon} \log \frac{3}{\epsilon} \rceil$ and $\rho = 1 - \frac{\epsilon}{2a}$.
Considering $n \geq 4a$, we get $\tau \geq (1 - \frac{\epsilon}{3a})^a \geq 1 - \frac{\epsilon}{3}$.
By (\ref{eq:noise-3}), the distance of our polynomial from $f$ is
$$ \left\| T_\rho f_{<d} - f\right\|_1 \leq
 \left(1 - \frac{\epsilon}{2a} \right)^d \|f\|_2 + 2 (1-\tau) \|f\|_2
  \leq \frac13 \epsilon \|f\|_2 + \frac23 \epsilon \|f\|_2
  \leq \epsilon \|f\|_2.$$
\end{proof}
}

\paragraph{Comparison of norms for self-bounding functions.}
Our bound on the noise operator implies a bound on the $\ell_1$ norm of a self-bounding function,
relative to its $\ell_\infty$ norm. This has been first shown for submodular functions by \citet{FeigeMV07} and for XOS functions by \citet{Feige:06} (with a constant $a$). \citet{FeldmanVondrak:16} show how this property together with approximation by a junta can be used to obtain a learning algorithm with multiplicative approximation guarantees that are required in the PMAC model of \citet{BalcanHarvey:12full}. Hence our results can be used to extend the PMAC learning algorithm in \citep{FeldmanVondrak:16} from XOS functions to all $a$-self-bounding functions. The full details of this construction are relatively involved and therefore we direct the interested reader to the discussion in \citep{FeldmanVondrak:16}(Sec.~6) for additional details.

\begin{lemma}
\label{lem:hypercontractive}
For any $a$-self-bounding function $f:\zo^n \rightarrow \RR_+$ under the uniform distribution,
with $n \geq 4a$,
$$ \|f\|_1 \leq \|f\|_\infty \leq 3^a \|f\|_1.$$
\end{lemma}

\begin{proof}
%Under our definition of norm by expectation, we always have $\|f\|_1^2 = \E[|f(x)|]^2 \leq \E[|f(x)|^2] = \|f\|_2^2$ by convexity of the square. Also, $\|f\|^2 \leq \max |f(x)| = \|f\|_\infty$.
%The interesting inequality is the last one.
Let $\|f\|_\infty = f(x^*)$. Since $f$ is nonnegative and $n \geq 4a$, we have by Lemma~\ref{lem:pointwise-bound}
$$ \|f\|_1 = \E[f(x)] = T_0 f(x^*) \geq \left(1 - \frac{1}{2(1-\frac{a-1}{n})}\right)^a f(x^*)
 \geq \left(1 - \frac{1}{2 \cdot 3/4}\right)^a f(x^*) = \frac{1}{3^a} f(x^*).$$
%On the other hand,
%$$ \|f\|^2_2 = \E[f^2(x)] \leq f(x^*) \E[f(x)] = f(x^*) \|f\|_1.$$
%We conclude that $\|f\|^2_2 \leq 3^{a} \|f\|^2_1$.
\end{proof}

We remark that a factor exponential in $a$ is necessary here. Consider the conjunction function on $a$ variables,
$f(x) = x_1 x_2 \cdots x_a$. This is an $a$-self-bounding function with values in $\{0,1\}$. We have
$$ \| f \|_p = ( \Pr[f(x) = 1] )^{1/p} = 2^{-a/p}.$$
In particular, $\| f \|_1 = 2^{-a}$, $\| f \|_2 = 2^{-a/2}$ and $\| f \|_\infty = 1$; i.e., the $\ell_1$, $\ell_2$ and $\ell_\infty$ norms can differ by factors exponential in $a$.

\paragraph{Relative error vs.~additive error.}
In our results, we typically assume that the values of $f(x)$ are in a bounded interval $[0,1]$ or that $\|f\|_1 \leq 1$ and our goal is to approximate $f$ with an additive error of $\epsilon$. As Lemma~\ref{lem:hypercontractive} shows, for $a$-self-bounding functions (with constant $a$) the $\ell_1$ and $\ell_\infty$ norms are within a bounded factor, so this does not make much difference.

%For an $a$-self-bounding function  $\|f\|_\infty \leq 3^a \|f\|_1$.
This means that if we scale $f(x)$ by $1 / (3^a \|f\|_1)$, we obtain a function with values in $[0,1]$. Approximating this function within an additive error of $\epsilon$ is equivalent to approximating the original function within an error of $\epsilon 3^a \|f\|_1$. In particular, for submodular functions we have $a=2$. Hence, the two settings are equivalent up to a constant factor in the error and we state our results for submodular functions in the interval $[0,1]$.

\subsection{Friedgut's theorem for $\ell_1$-approximation}
\label{sec:friedgut-rv}
As we have mentioned in Lemma \ref{lem:submod}, self-bounding functions have low total sensitivity. A celebrated result of \citet{Friedgut:98} shows that any Boolean function on $\zo^n$ of low average sensitivity is close to a function that depends on few variables. His result was extended to $\ell_2$ approximation of real-valued functions in \citep{FeldmanVondrak:16}. We now show that for self-bounding functions a tighter bounds can be achieved for $\ell_1$ approximation. Our proof is based on the use of $\ell_1$ approximation by polynomials proved in Theorem \ref{thm:selfbounding-approx} together with the analysis from \citep{FeldmanVondrak:16} to obtain a smaller $\ell_1$ approximating junta.

We now state the main result in more detail.
\begin{theorem}
\label{th:selfbound-junta-l1-bound}
Let $f:\zo^n \rightarrow [0,1]$ be a function and $a = \infl^1(f)$. For every $\eps>0$, let $d = \lceil \frac{4a}{\epsilon} \log \frac{4}{\epsilon} \rceil$ and $I = \{i\in[n] \cond \infl^{4/3}_i(f) \geq \alpha\}$ for $\alpha = 3^{-2d-1} \eps^4/a^2 $. Then $|I| \leq a/\alpha$ and there exists a
polynomial $p$ of degree $d$ over variables in $I$ such that $\|f - p\|_1 \leq \epsilon.$
\end{theorem}

To prove the theorem we will need the following bound on the sum of squares of all low-degree Fourier coefficients that include a variable of low influence from \citep{FeldmanVondrak:16}.

\begin{lemma}[\citep{FeldmanVondrak:16}, Lemma 4.7]
\label{lem:low-influence-bound-general}
Let $f:\zo^n\rightarrow \RR$, $\kappa \in (1,2)$, $\alpha >0$ and $d$ be an integer $\geq 1$. Let $I = \{i\in[n] \cond \infl^{\kappa}_i(f) \geq \alpha\}$. Then $$\sum_{S\not\subseteq I, |S|\leq d}\hat{f}(S)^2 \leq (\kappa-1)^{1-d} \cdot \alpha^{2/\kappa-1} \cdot \infl^{\kappa}(f)\ .$$
\end{lemma}

We can now complete the proof of Thm.~\ref{th:selfbound-junta-l1-bound}.

\begin{proof}
Theorem \ref{thm:selfbounding-approx} proves that for $d \leq \lceil \frac{4a}{\epsilon} \log \frac{4}{\epsilon} \rceil$ and $\rho = 1-\frac{\eps}{2a}$, the function $T_\rho f_{<d}$
satisfies \equ{\|f - T_\rho f_{<d}\|_1 \leq \eps \|f\|_2/2 \leq \eps/2\label{eq:low-deg-diff}.}

We can also apply Lemma \ref{lem:low-influence-bound-general} with $\kappa = 4/3$
and $\alpha = 3^{-2d-1} \eps^4/a^2$ to obtain that
\equ{\sum_{S\not\subseteq I, |S|\leq d}\hat{f}(S)^2 \leq 3^{d-1} \cdot \alpha^{1/2} \cdot \infl^{4/3}(f) =
3^{d-1} \cdot  \left(3^{-d-1/2} \cdot \frac{\eps^2}{a} \right) \cdot \infl^{4/3}(f) \leq \frac{\eps^2}{4}\ ,\label{eq:bound-low-inf}}
where the last inequality uses $\infl^{4/3}(f) \leq \infl^1(f) \leq a$ which follows from the fact that $\partial_i f$'s have range $[-1/2,1/2]$ when $f$ has range $[0,1]$.

For every $S$, $|\widehat{T_\rho f}(S)| = |\rho^{|S|} \hat{f}(S)| \leq |\hat{f}(S)|$. Therefore eq.~\eqref{eq:bound-low-inf} implies that \equ{\sum_{S\not\subseteq I, |S|\leq d}\widehat{T_\rho f}(S)^2 \leq \sum_{S\not\subseteq I, |S|\leq d}\hat{f}(S)^2 \leq \frac{\eps^2}{4}\ .\label{eq:bound-low-diff} }
Now let $p=\sum_{S \subseteq I,\ |S|\leq d} \widehat{T_\rho f}(S)\chi_S$ be the restriction of $T_\rho f_{<d}$ to variables in $I$.
Equation (\ref{eq:bound-low-diff}) gives a bound on the sum of squares of all the coefficients that we removed from $T_\rho f_{<d}$ and implies that $\|p- T_\rho f_{<d}\|_1 \leq \|p- T_\rho f_{<d}\|_2 \leq \eps/2$.
Together with eq.~\eqref{eq:low-deg-diff}, we get $\|f - p\|_1 \leq \eps$.
Finally, $|I| \leq \infl^{4/3}(f)/\alpha \leq \infl^1(f)/\alpha \leq a/\alpha$.
\end{proof}

By Lemma \ref{lem:submod}, every $a$-self-bounding function $f:\zon \to [0,1]$ satisfies, $\infl^1(f) \leq a$. Hence as an immediate corollary we obtain Thm.~\ref{th:selfbound-junta-l1-bound-simple}. Another immediate corollary of Thm.~\ref{th:selfbound-junta-l1-bound} is that for every $a$-self-bounding function there exists a polynomial of low total $\ell_1$-spectral norm that approximates it.
\begin{corollary}
\label{cor:selfbound-junta-l1-spectral-bound}
Let $f:\zo^n \rightarrow [0,1]$ be an $a$-self-bounding function and $\eps>0$. There exist $d = O(a/\epsilon \cdot \log (1/\epsilon))$ and a
polynomial $p$ of degree $d$ such that $\|f - p\|_1 \leq \epsilon$ and $\|\hat{p}\|_1= 2^{O(d^2)}$, where $\|\hat{p}\|_1 = \sum_{S \subseteq [n]} |\hat{p}(S)|.$
\end{corollary}

%Our claim is scale invariant so for simplicity we assume that $\|f(x)\|_\infty = 1$.
\section{Algorithmic applications}
\label{sec:applications}
We now outline the applications of our structural results. They are based on using our stronger bounds in existing learning algorithms for submodular, XOS and self-bounding functions. We omit the details of the algorithms and their analysis since they follow closely those of the corresponding results in \citep{FeldmanVondrak:16}.
\subsection{Learning Models}
Our learning algorithms are in one of two standard models of learning. The first one assumes that the learner has access to random examples of an unknown function from a known set of functions. This model can be seen as a generalization of Valiant's PAC learning model to real-valued functions \citep{Valiant:84}. While in general Valiant's model does not make assumptions on the distribution $\D$, here we only consider the {\em distribution-specific} version of the model in which the distribution is fixed and is uniform over $\zo^n$.
\begin{definition}[Distribution-specific $\ell_1$ PAC learning]
Let $\F$ be a class of real-valued functions on $\zo^n$ and let $\D$ be a distribution on $\zo^n$. An algorithm $\A$ PAC learns $\F$ on $\D$, if for every $\epsilon > 0$ and any target function $f \in \F$, given access to random  independent samples from $\D$ labeled by $f$, with probability at least $\frac{2}{3}$,  $\A$ returns a hypothesis $h$ such that $\E_{x \sim \D} [ |f (x) - h(x) | ] \leq  \epsilon.$
\end{definition}
%The error parameter $\eps$ in the Boolean case measures probability of misclassification.

Agnostic learning generalizes the definition of PAC learning to scenarios where one cannot assume that the input labels are consistent with a function from a given class \citep{Haussler:92,KearnsSS:94} (for example as a result of noise in the labels).
\begin{definition}[Distribution-specific $\ell_1$ agnostic learning]
Let $\F$ be a class of real-valued functions on $\zo^n$ and let $\D$ be any fixed distribution on $\zo^n$. For any distribution $\D'$, let $\mbox{opt}(\D',\F)$ be defined as: $$\mbox{opt}(\D',\F) =  \inf_{f \in \F} \E_{(x,y) \sim \D'} [ |y - f(x) |] .$$ An algorithm $\A$, is said to agnostically learn $\F$ on $\D$ if for every {\em excess error} $\epsilon> 0$ and any distribution $\D'$ on $\zo^n \times \R$ such that the marginal of $\D'$ on $\zo^n$ is $\D$, given access to random independent examples drawn from $\D'$, with probability at least $\frac{2}{3}$, $\A$ outputs a hypothesis $h$ such that $\E_{(x,y) \sim \D'} [ |h(x)- y| ] \leq \mbox{opt}(\D') + \epsilon.$
\end{definition}

The first corollary of our structural results is for PAC learning of monotone self-bounding functions (the results also apply to {\em unate} functions which are either monotone or anti-monotone in each variable). Note that this class of functions includes XOS functions.
\begin{theorem}
\label{thm:pac-learn-monotone}
Let $\C^+_a$ be the set of all monotone $a$-self-bounding functions on from $\zo^n$ to $[0,1]$. There exists an algorithm that PAC learns $\C^+_a$ over the uniform distribution, runs in time $\tilde{O}(n) \cdot 2^{\tilde{O}(a^2/\eps^{2})}$ and uses $2^{\tilde{O}(a^2/\eps^{2})} \log n$ examples, where $\eps$ is the error parameter.
\end{theorem}
The proof of this result follows from substituting our bounds in Theorems \ref{th:selfbound-junta-l1-bound-simple} and  \ref{cor:selfbound-junta-l1-spectral-bound} into the simple analysis from \citep{FeldmanVondrak:16}.

Our main application to agnostic learning is the algorithm for learning self-bounding functions from random examples described in Theorem \ref{th:agn-learn-self-bound-l1-simple}. The algorithm used to prove this result is again polynomial $\ell_1$ regression over all monomials of degree $\tilde{O}(a/\eps)$. In addition, we can rely on the existence of a polynomial of low spectral norm to obtain substantially tighter bounds on sample complexity. Namely, as in \citep{FeldmanVondrak:16}, we use the uniform convergence bounds for linear combinations of functions with $\ell_1$ constraint on the sum of coefficients \citep{KakadeST:08}.
Our structural results also have immediate implications for learning with value queries, that is oracle access to the value of the unknown function at any point $x$. Following the approach from \citep{FeldmanKV:13}, we can use the algorithm of \citet{GopalanKK:08} together with our bounds on the spectral norm of the approximating polynomial in Cor.~\ref{cor:selfbound-junta-l1-spectral-bound}. This leads to the following algorithm.
%The we get a $\poly(n) \cdot 2^{\tilde{O}(a^4/\eps^{2})}$ time agnostic learning algorithm for all $a$-self-bounding functions and a $\poly(n) \cdot 2^{\tilde{O}(1/\eps)}$ algorithm for agnostic learning of submodular functions.
\begin{theorem}
\label{th:agn-learn-self-bound-valueq}
Let $\C_a$ be the class of all $a$-self-bounding functions from $\zo^n$ to $[0,1]$.  There exists an agnostic learning algorithm that for any $\eps > 0$, given access to value queries learns $\C_a$ with excess error $\eps >0$ over the uniform distribution in time $\poly(n) \cdot 2^{\tilde{O}(a^2/\eps^2)}$.
\end{theorem}

\section{Lower bounds for learning self-bounding functions}
\label{sec:lower-bounds}
In this section, we show that learning $a$-self bounding functions within an error of at most $\epsilon$, is at least as hard as learning the class of all DNFs (of any size) of width at most $\lfloor \frac{a}{4\epsilon} \rfloor$ to an accuracy of $\frac{1}{4}$.
%The class of $k$-DNF formulas includes all $k$-juntas (boolean functions that depend on just $k$ of the total $n$ variables). Learning $k$-juntas is a long standing open problem posed by Blum and Langley \citep{BlumLangley:97,Blum:03op}. An influential paper of Mossel, O'Donnell and Servedio \citep{MosselOS:04} gave an algorithm that takes $O(n^{0.7k})$ for the problem. Recently, Valiant \citep{ValiantG:12} gave an algorithm improves on this bound and learns the class of $k$-juntas in time $O(n^{0.6k})$.
\eat{
Our reduction here, will thus imply that any algorithm that learns self-bounding functions within an error of at most $\epsilon$ in time $n^{o(\frac{1}{\epsilon})}$ time will give an algorithm that beats the state of the art algorithm for the junta problem and give an algorithm that learns $k$-juntas in time $n^{o(k)}$.
}
% The main idea in the reduction is the following construction of a self bounding function that encodes a $k$-DNF. Given a $k$-DNF $c$ on $n$ variables, we will construct a self bounding function $f$ on $n$ Boolean variables which ``encodes" $c$. That is, given any approximator for $f$, we will be able to obtain an approximator for $c$ with some loss in the accuracy. For this purpose, we will define $f$ to take the value $1$ wherever $c$ takes the value $1$ and value $1-\frac{1}{k}$ wherever $c$ takes the value $0$. To show that $f$ is self-bounding, we need to show that $\forall x \in \on^n \text{, } f(x) \geq \sum_{i =1}^n (f(x) -\min_{x_i} f(x)).$ This follows from analyzing the cases where $f$ takes each of the possible two values separately and noticing that whenever $c$ takes the value $1$, there can be at most $k$ variables which when flipped, change the value of $c$.
Our reduction to learning width $k$-DNFs (also referred to as $k$-DNFs) is based on the simple observation that $k$-DNFs are $k$-self bounding functions combined with a simple linear transformation that reduces approximation and learning of $(a\cdot r)$-self bounding functions for $r\geq 1$ to that of $a$-self-bounding functions.

\begin{lemma}
A function $f:\zo^n \rightarrow \zo$ computed by a $k$-DNF formula is a $k$-self bounding function. \label{lem:DNF-self-bounding}
\end{lemma}
\begin{proof}
%We verify $k$-DNFs satisfy the definition of $k$-self bounding functions.
Since $f$ is $\zo$-valued, clearly, $f(x) - \min_{x_i} f(x) \leq 1$ for any $i \in [n]$. If $f(x) = 0$, then, $\sum_{i = 1}^n (f(x) - \min_{x_i} f(x)) = 0 \leq k \cdot f(x)$. Now suppose $f(x) = 1$. Then, there exists at least one term, say $T$, of the DNF that is satisfied by the assignment $x$. Observe that if we flip a literal outside of $T$, then, the value of $f$ remains unchanged. Thus, if the term indexed by $j$ in $\sum_{i = 1}^n (f(x) - \min_{x_i} f(x))$ contributes the value $1$, then either $x_j \in T$ or $\bar{x}_j \in T$. In particular, at most $k$ terms in the sum contribute $1$ and the rest contribute $0$. Thus,  $\sum_{i = 1}^n (f(x) - \min_{x_i} f(x)) \leq k = k\cdot f(x).$
\end{proof}

\begin{remark}
In light of Lemma \ref{lem:DNF-self-bounding} it is natural to ask whether all Boolean $k$-self-bounding functions are $k$-DNF.
It is easy to see that for Boolean functions being $k$-self-bounding can be equivalently stated as having 1-sensitivity of $k$. The smallest $k$ for which $f$ can be represented by a $k$-DNF is referred to as $1$-certificate complexity of $f$. It has long been observed that for monotone functions $1$-certificate complexity equals $1$-sensitivity \citep{Nisan:1989} and therefore all monotone $k$-self-bounding functions are $k$-DNF. However this is no longer true for non-monotone functions. A simple example in \citep{Nisan:1989} gives a function with a factor two gap between these two measures. Quadratic gap for every $k$ up to $\Theta(n^{1/3})$ is also known \citep{Chakraborty:05}.
\end{remark}

Next, we observe that for any $a$-self-bounding function, the function $g$ defined by $g(x)  = 1 - \frac{1}{r} + \frac{f(x)}{r}$ is $\frac{a}{r}$-self-bounding whenever $r \geq 1$. This ``lifting'' transforms an $a$-self-bounding functions into an $\frac{a}{r}$-self-bounding functions.

\begin{lemma} \label{lem:lifting-trick}
Let $f:\zo^n \rightarrow [0,1]$ be an $a$-self-bounding function. Then for any $r \geq 1$,  $g(x)  = 1 - \frac{1}{r} + \frac{f(x)}{r}$ has range $[0,1]$ and is $\frac{a}{r}$-self-bounding.
\end{lemma}
\begin{proof}
Clearly, the $1-1/r+f(x)/r$ transformation maps $[0,1]$ to $[1-1/r,1] \subseteq [0,1]$. Observe that for any $x$ and $i \in [n]$, $g(x) - \min_{x_i} g(x) = \frac{1}{r} \cdot (f(x) - \min_{x_i} f(x))$ and also that $g(x) \geq f(x)$. By the definition of $a$-self-boundedness we obtain that $g$ is $a/r$-self bounding.
\end{proof}

Observe that given random examples labeled by $f$, it is easy to simulate random examples labeled by $g$. Further, $\ell_1$-approximation of $f$ within $\eps$ can be translated (via the same ``lifting") to $\eps/r$-approximation of $g$ and vice versa. An immediate corollary of this is that one can use a learning algorithm for $a/r$-self-bounding functions to learn $a$-self bounding functions.  We use $\C_a^n$ to denote the class of all $a$-self-bounding functions from $\zo^n$ to $[0,1]$.

%we see that if a learner uses random examples labeled by $g$ to produce a hypothesis $h$ that is $\epsilon$-close in $\ell_1$ distance to $g$, then, $h'$ defined by $h'(x) = 1 - r(1-g(x))$ is $r \cdot \epsilon$-close in $\ell_1$ distance to $f$.

\begin{lemma}
\label{lem:PAC-learn-transform}
Let $a\geq 1$ and $a \geq  r \geq 1$. Suppose there is an algorithm that PAC (or agnostically) learns $\C_{a/r}^n$ over a distribution $D$ with $\ell_1$ error of $\epsilon$ in time $T(n,1/\epsilon)$. Then, there is an algorithm that PAC (or, respectively, agnostically) learns $\C_{a}^n$ over $D$ with $\ell_1$ error of $\epsilon$ in time $T(n, 1/(r\epsilon))$.
\end{lemma}
\eat{
\begin{proof}
Let $f: \zo^n \rightarrow \R$ be the target $(a \cdot r)$-self bounding function.  For every random example $(x,y)$ such that $y= f(x)$, we create a random example $(x,y'$) such that $y' = 1 - (1-y)/r$. Then, $y' = g(x)$ for $g = 1 - \frac{1-f(x)}{r} $ and is thus $a$-self bounding using Lemma \ref{lem:lifting-trick}. We now run the PAC learning algorithm with error at most $\epsilon/r$ for $a$-self bounding functions on the random examples labeled by $g$ as above. The algorithm (with probablility at least $2/3$) returns a hypothesis $h$ such that $\E_{x \sim D} [ |h(x) - g(x)|]  \leq \epsilon/r$. Set $h'(x) = 1 - r(1-h(x))$. Then, $\E_{x \sim D}[ |h'(x) - f(x)|] = \E_{x \sim D}[ r|(1-h(x) - (1-g(x)|] \leq r \cdot \epsilon/r = \epsilon$. This completes the proof.
\end{proof}
}

The simple structural observations above give us our lower bounds for learning and approximation of $a$-self-bounding functions.  Using Lemmas \ref{lem:DNF-self-bounding} and \ref{lem:PAC-learn-transform}, we have the the following lower bound on the time required to PAC learn $a$-self-bounding functions.
\begin{theorem}[Thm.~\ref{th:self-bound-pac-is-hard} restated]
Suppose there exists an algorithm that PAC learns $\C_{a}^n$ with $\ell_1$ error of $\eps >0$ with respect to the uniform distribution in time $T(n,1/\eps)$. Then, for any $k \geq a$, there exists an algorithm that PAC learns $k$-DNF formulas with disagreement error of at most $\eps'$ with respect to the uniform distribution in time $T(n, \frac{k}{a\eps'})$. Consequently, there exists an algorithm for learning $k$-juntas on the uniform distribution to an error of at most $1/4$ in time $T(n, \frac{k}{4a})$ for any $k \geq a$.
\end{theorem}

Now, $k$-juntas contain the set of all Boolean functions on any fixed subset of $k$ variables. A standard information-theoretic lower bound implies that any algorithm that PAC learn $k$-juntas to an accuracy of $1/4$ on the uniform distribution needs $\Omega(2^{k})$ random examples or even value queries. This translates into the following unconditional lower bound for learning $a$-self-bounding functions.
\begin{corollary}
\label{cor:inf-lower-bound-sb}
Any algorithm that PAC learns $\C_a$ over the uniform distribution needs $\Omega(2^{a/\epsilon})$ random examples or value queries.
\end{corollary}

Finally, observe that the $\zo$-valued parity function on $k$ bits is computed by a $k$-DNF formula and any polynomial that $1/4$-approximates in $\ell_1$ distance on the uniform distribution must have degree at least $k$.
Thus, we have the following degree lower bound for polynomials that $\ell_1$ approximate $a$-self-bounding functions on the uniform distribution on $\zo^n$.

\begin{corollary}
\label{cor:degree-lower-bound-sb}
Fix an $a \geq 1$ and $\epsilon \in (0,1/4]$. There exists an $a$-self-bounding function $f:\zo^n \rightarrow [0,1]$, such that every  polynomial $p$ that $\eps$-approximates $f$ in $\ell_1$ norm with respect to the uniform distribution has degree $d \geq a/(4\epsilon)$.
\end{corollary}
\begin{proof}
Let $k =  \frac{a}{4\epsilon}$ (ignoring rounding issues for simplicity) and $f$ be a $\zo$-valued parity on some set of $k$ variables. By Lemma \ref{lem:DNF-self-bounding} $f$ is $k$-self-bounding. Then, as in the proof of Lemma \ref{lem:PAC-learn-transform}, for $r = \frac{1}{4\epsilon}  \geq 1$, $g$ defined by $g(x) = 1 - \frac{1-f(x)}{r}$ is an $a$-self-bounding function. Let $p$ be a polynomial of degree $d$ that approximates $g$ within an $\ell_1$ error of $\epsilon$ with respect to the uniform distribution on $\zo^n$. Then, as in the proof of Lemma \ref{lem:PAC-learn-transform},  $p' = 1 - r(1-p)$ is a polynomial of degree $d$ and approximates $f$ within an $\ell_1$ error of at most $\frac{1}{4\epsilon} \cdot \epsilon = 1/4$.

For the $\on$-valued parity $\chi = 2f(x)-1$ and any polynomial $p'$ of degree less than $k$, $\E[\chi \cdot p']=0$. Further, $\E[|\chi - p'|] \geq 1-\E[\chi \cdot p'] = 1$. This implies that for $f$ the  $\ell_1$ error of any polynomial of degree at most $k-1$ is at least $1/2$. In particular, $d \geq a/(4\epsilon)$.
\end{proof}

We remark that slightly weaker versions of Cor.~\ref{cor:inf-lower-bound-sb} and Cor.~\ref{cor:degree-lower-bound-sb} are known for monotone submodular functions. Specifically, they require $2^{\Omega(\epsilon^{-2/3})}$ random examples or value queries to PAC learn and also degree $\Omega(\epsilon^{-2/3})$ to approximate \citep{FeldmanKV:13}.% This suggests a natural open problem of whether the lower bounds for submodular functions can be strengthened to $2^{\Omega(1/\epsilon)}$ examples and $\Omega(1/\epsilon)$ degree.

\printbibliography

%%--------------------------------------------------------------------}
\end{document}